%% file: arxiv-version.tex
\documentclass[11pt]{article}

\usepackage{amsmath, amssymb, amsthm, mathtools}
\usepackage{times}
\usepackage{fullpage}
\usepackage{hyperref}
\usepackage{algorithm}
\usepackage{algorithmic}

\title{The Hidden Game Problem}

\author{
  Gon Buzaglo \thanks{Princeton University,  \texttt{gon.buzaglo@princeton.edu}}
  \and
  Noah Golowich \thanks{Microsoft Research,  \texttt{nzg@mit.edu}}
  \and
  Elad Hazan \thanks{Princeton University, \texttt{ehazan@princeton.edu}}
}

\date{} 

\input{header-for-arxiv}

\begin{document}

\maketitle

\begin{abstract}
This paper investigates a class of games with large strategy spaces, motivated by challenges in AI alignment and language games. We introduce the \emph{hidden game problem}, where for each player, an unknown subset of strategies consistently yields higher rewards compared to the rest. The central question is whether efficient regret minimization algorithms can be designed to discover and exploit such hidden structures, leading to equilibrium in these subgames while maintaining rationality in general. We answer this question affirmatively by developing a composition of regret minimization techniques that achieve optimal external and swap regret bounds. Our approach ensures rapid convergence to correlated equilibria in hidden subgames, leveraging the hidden game structure for improved computational efficiency.
\end{abstract}

\input{1-intro}

\input{2-hidden_game}
\input{3-det-alg-sim-ext-swap-new}
\input{6-discussion}


\bibliographystyle{plainnat}
\bibliography{main}

\appendix

\input{A1-deferred-proofs}

\end{document}

%% file: header-for-arxiv.tex
\usepackage{amsmath,amsfonts,amssymb}
\usepackage{mathtools}
\usepackage{amsthm}            
\usepackage{bm}                
\usepackage{nicefrac}          

\usepackage{latexsym}
\usepackage{relsize}
\usepackage{xcolor}
\usepackage{dsfont}

\usepackage{thmtools,thm-restate}   

\usepackage[capitalise]{cleveref}

\usepackage[authoryear,round,sort&compress]{natbib}
\bibpunct{(}{)}{;}{a}{}{,}
\setcitestyle{aysep={,}} 

\numberwithin{equation}{section}

\theoremstyle{plain} 
\newtheorem{theorem}{Theorem}[section]
\newtheorem{lemma}[theorem]{Lemma}
\newtheorem{corollary}[theorem]{Corollary}

\newtheorem{fact}[theorem]{Fact}
\newtheorem{observation}[theorem]{Observation}
\newtheorem{problem}[theorem]{Problem}

\theoremstyle{definition} 
\newtheorem{definition}[theorem]{Definition}

\theoremstyle{remark} 
\newtheorem{remark}[theorem]{Remark}

\newtheorem*{theorem*}{Theorem}
\newtheorem*{lemma*}{Lemma}
\newtheorem*{corollary*}{Corollary}
\newtheorem*{proposition*}{Proposition}
\newtheorem*{claim*}{Claim}
\newtheorem*{fact*}{Fact}
\newtheorem*{observation*}{Observation}
\newtheorem*{assumption*}{Assumption}
\newtheorem*{definition*}{Definition}
\newtheorem*{remark*}{Remark}
\newtheorem*{example*}{Example}

\newcommand{\phiregret}[1][\Phi]{{#1}\text{-Regret}}

\newcommand{\BR}{\text{BR}}

\def\calH{{\mathcal H}}
\def\calE{{\mathcal E}}

\newcommand{\mB}{\mathcal{B}}

\newcommand{\A}{\mathcal{A}}

\def\calB{{\cal  B}} 

\newcommand{\ignore}[1]{}

\renewcommand{\eqref}[1]{Eq.~(\ref{#1})}
\newcommand{\lemref}[1]{Lemma~\ref{#1}}

\renewcommand{\le}{\leqslant}
\renewcommand{\ge}{\geqslant}
\renewcommand{\leq}{\le}
\renewcommand{\geq}{\ge}

\DeclareMathAlphabet{\mathbfsf}{\encodingdefault}{\sfdefault}{bx}{n}

\newcommand{\mycases}[4]{{
\left\{
\begin{array}{ll}
    {#1} & {\;\text{#2}} \\[1ex]
    {#3} & {\;\text{#4}}
\end{array}
\right. }}

\renewcommand{\O}{O}

\newcommand{\E}{\mathbb{E}}

\newcommand{\poly}{\mathrm{poly}}

\renewcommand{\leq}{~\le~}
\renewcommand{\geq}{~\ge~}

\let\oldtfrac\tfrac
\renewcommand{\tfrac}[2]{\smash{\oldtfrac{#1}{#2}}}

\let\nablaold\nabla
\renewcommand{\nabla}{\nablaold\mkern-2.5mu}

\newcommand{\Aswap}{\mathcal{A}_{\mathsf{swap}}}
\newcommand{\swapregret}{\text{SwapRegret}}


%% file: 1-intro.tex
\section{Introduction}

The rapid progress of large language models has created new game-theoretic challenges in AI alignment. 
Settings such as \emph{AI debate}~\citep{irving2018debate} naturally involve 
\emph{exponentially large action spaces}: each action may correspond to an entire argument, or even a 
sequence of sentences. Among this enormous set of strategies, only a tiny fraction are meaningful: for 
example, grammatical and relevant arguments. This motivates the following question:
\textit{How can agents efficiently discover and exploit sparse hidden structure in very large games, 
while still maintaining rationality guarantees in the full game?}

We formalize this as the \emph{hidden game problem}: each player has an unknown subset of strategies 
$R \subset [N]$, with $|R| = r \ll N$, whose rewards are consistently larger than those outside $R$. 
The challenge is to design algorithms that: 
\begin{enumerate}
    \item \textbf{Behave rationally in general}, by achieving low external regret regardless of the presence of hidden structure; and
    \item \textbf{Exploit hidden structure when it exists}, by achieving low swap regret with complexity that depends only on $r$, not $N$.
\end{enumerate}
This setting captures a key alignment intuition: while the ambient strategy space may be enormous, 
only a sparse hidden set of strategies leads to coherent, high-quality interaction.

We formalize and give efficient algorithms for the hidden game problem. Our main contributions are:

\begin{itemize}
    \item \textbf{Model.} We introduce a formal definition of hidden games, where payoffs decompose as
\[
    A \;=\; A_0 + \rho A_1,
\]
with $A_0$ enforcing that all strategies in $R$ dominate those outside.  
The learner’s objective is twofold: (i) achieve low swap regret when a hidden game exists, and (ii) guarantee low external regret in any case.

    \item \textbf{Swap regret minimization in hidden sets.} We design an algorithm that incrementally uncovers the hidden set $R$, 
    and prove it achieves swap regret scaling as
    \[
        O\!\left(\sqrt{T r^3 \log r}\right),
    \]
    with per-iteration runtime independent of $N$. To our knowledge, this is the first sublinear swap-regret 
    bound that depends on the hidden set size $r$, not on the ambient action space $N$.

    \item \textbf{Simultaneous external + swap regret.} We further develop a combined algorithm that 
    (i) achieves the standard external regret guarantee $O(\sqrt{T \log N})$ in \emph{all games}, and 
    (ii) achieves low swap regret whenever a hidden set exists. This uses a novel composition of Hedge, 
    Follow-the-Perturbed-Leader with smooth optimization oracles and a fixed-point update scheme.
\end{itemize}

\begin{theorem}[\cref{thm:combined}, informal]\label{thm:main-informal}
    There is an algorithm that achieves, with probability at least $1-\delta$ and against a fully adaptive adversary that selects an arbitrary sequence of loss vectors $\ell_t \in [0,1]^N$, external regret of
    \begin{align*}
        \max_{i \in [N]} \sum_{t=1}^T \ell_t(i) - \sum_{t=1}^T \ell_t^\top x_t = O\!\left(\sqrt{T \log N}\right)\,,
    \end{align*}
    and if there is a hidden game $R \subset [N]$, swap regret of
    \begin{align*}
        \max_{\phi \in \Phi_S} \sum_{t=1}^T \ell_t^\top \phi(x_t) - \sum_{t=1}^T \ell_t^\top x_t = O\!\left(\sqrt{T r^3 \log r}\right)\,,
    \end{align*}
    where $\Phi_S$ is the set of all fixed deviations. Furthermore, with access to a smooth optimization oracle (see \cref{subsec:oracles}), the running time of each iteration is $\mathrm{Poly}\!\left(T, \log\frac{1}{\delta}\right)$ and independent of $N$.
\end{theorem}

\begin{table}[h]
\centering
\small
\setlength{\tabcolsep}{5pt}
\renewcommand{\arraystretch}{1.1}
\begin{tabular}{l|c|c|c|c}
 & running time & external regret & swap regret & oracle-based \\ \hline
\cite{blummansour2007fromexternaltointernal} & $N^{3}$ & $\sqrt{NT \log N}$ & $\sqrt{NT \log N}$ & $\times$ \\
\cite{dagan2023external} & $N$ & $T \cdot \log \log N/\log T $ & $T \cdot \log \log N/\log T $ & $\checkmark$ \\ 
\cite{chen2024playing} & $\mathrm{poly}(T)$ & $\sqrt{T \log N}$ & $\rho T + r\sqrt{T\log N}$ & $\checkmark$ \\
\textbf{Ours (\cref{thm:combined})} & $\mathrm{poly}(T)$ & $\sqrt{T \log N}$ & $\sqrt{T r^3 \log r}$ & $\checkmark$
\end{tabular}
\caption{Comparison of regret guarantees for large games. Our algorithm is the only one to simultaneously minimize external regret in all games and achieve low swap regret in hidden games, with per-round runtime $\mathrm{poly}(T)$ independent of $N$ (in contrast to prior work such as \cite{peng2023fast} and \cite{dagan2023external}, which scales polynomially with $N$). The swap regret of \citet{chen2024playing} is low only if $\rho = o(1)$, whereas for the hidden game problem we typically take $\rho$ to be a constant.}
\end{table}

\subsection{Related Work}

\paragraph{Solution concepts in game theory and notions of equilibria.} Nash equilibrium has emerged as the central solution concept in game theory. It characterizes a stable solution to the game, by a situation in which each player independently draws her strategy from her own probability distribution, and then no player can benefit from being the only one deviating from her distribution. \cite{nashnoncooperativegames} proved that in any game with a finite set of players and a finite set of actions there exists a Nash equilibrium. It was shown by \cite{settlingnash, daskalakis2009complexity} that computing a Nash equilibrium is as hard as computing a fixed point for a general mapping, which is a \(\textsf{PPAD}\)-complete problem, as shown by \cite{10.1016/S0022-0000(05)80063-7}. To overcome this, the concept of correlated equilibrium (CE) \cite{aumann1974subjectivity} was suggested as a relaxation to Nash equilibrium, where the players' distributions are not independent. 
For further reading on notions of equilibria, see \cite{NisaRougTardVazi07}. The advantage of CE, as we will discuss immediately, is that not only it can be efficiently computed, but also through "natural game play" of independent players. 

\paragraph{Learning in games and notions of regret} Regret minimization is a framework that models sequential decision making in the face of uncertainty and is surveyed in detail in \cite{cesa2006prediction,hazan2016introduction}. The idea is to design an algorithm such that the decision maker payoffs will not be too far from those incurred by some simple modification of her policy. Different notions of regret correspond to different definitions of alternative policies, and minimizing each type has been shown to imply convergence to a corresponding notion of equilibrium. \cite{hannan1957approximation} gave algorithms to minimize external regret, and \cite{freund1999adaptive} show that in a zero-sum game, if all agents use a no external regret algorithm, then
the empirical distribution of the play converges to a Nash equilibrium.
For a general sum game, Minimizing external regret guarantees convergence only to coarse correlated equilibria, which can be strictly larger than the set of correlated equilibria. Thus, \cite{foster1997calibrated} introduce the notion of internal regret, and \cite{Hart00asimplecorrelated} proved that if all players minimize follow an algorithm that minimizes their internal regret, then the joint empirical distribution of the players' actions converges to a CE.

\paragraph{Swap regret minimization and large games.} 
When the number of actions \(N\) is very large (exponential in the game horizon \(T\)) it is trivial to minimize internal regret alone, for example by the uniform distribution. Thus, external regret is still of interest in that regime, to model rational behavior of players. Furthermore, the rate of convergence to CE when minimizing internal regret depends linearly on \(N\), which is not practical in the setting of a large game. 
To overcome those limitations, Blum and Mansour \cite{blummansour2007fromexternaltointernal} introduced the concept of swap regret, which generalizes both external and internal regret, and showed how to reduce swap regret minimization to external regret minimization. The stronger nature of swap regret leads it to have numerous applications, including in calibration \cite{globus-harris2023multicalibration,kleinberg2023ucalibration} and conformal prediction \cite{ramalingamrelationship}, among others. However, existing no-swap-regret algorithms are not useful when the number of possible strategies is very large. The classical reductions of \cite{stoltz2005internal,blummansour2007fromexternaltointernal} yield swap regret of \(\sqrt{TN\log N}\), with per-iteration running time polynomial in \(N\). 
Another motivation to minimize swap regret, of major interest in multi-agent settings, is that it is known to be more robust to strategic manipulation than external regret \cite{mansour2022strategizing}. 
Recent works by \cite{peng2023fast,dagan2023external} improved the regret bound, while their running time is still polynomial in \(N\).In contrast to these works, our algorithm achieves regret bounds that depend only on the hidden set size $r$ while maintaining per-round runtime polynomial in $T$ and independent of $N$.
The hidden game problem, as described in \cref{subsec:hidden_game}, was also foreshadowed by \cite{chen2024playing}, who proposed an algorithm for simultaneous external and internal regret minimization, with the hidden-game setting appearing as a special case when the rewards of actions outside the hidden set are $o(T)$. In contrast, we obtain a result that holds for any fixed $\rho \in (0,1)$.

\paragraph{Prediction from Structured Experts Advice.}
A large body of work has studied prediction with expert advice under structural assumptions. Closely related to our setting, \cite{chernov2010predictionadviceunknownnumber} and \cite{pmlr-v30-Gofer13} design bounds that scale with the effective number of experts or the realized complexity of the sequence, rather than the nominal size of the expert class, while ~\cite{hazan2016onlinelearninglowrank} show that regret can depend only on the latent rank of the loss matrix. Another important line of work addresses \emph{tracking the best expert} in dynamic environments. Beginning with ~\cite{bousquet2002tracking}, algorithms were developed to compete with the best sequence of experts over time, and ~\cite{mourtada2017efficienttrackinggrowingnumber} studied efficient tracking when new experts arrive dynamically. Our approach is conceptually related to these ideas, as we also maintain a growing set of candidate experts to approximate the hidden structure of the game. 

\paragraph{Learning with a Double Oracle.}
Our approach can also be viewed as a variant of the double oracle algorithm \cite{mcmahan2003planning}, where players iteratively add best-response strategies to a restricted game. This framework is widely used in multi-agent reinforcement learning under the name policy-space response oracles; see, for example, \cite{muller2021learning, mcaleer2021xdo, gemp2022developing, bighashdel2024policy}. A closely related perspective was recently adopted by \cite{assos2023online}, who studied minimizing external regret in infinite games with small Littlestone dimension. However, we focus on swap regret, for which such dimension is not known.

\paragraph{Simultaneous Regret Minimization.}
Many works in online learning combine different regret minimization algorithms to run simultaneously. \cite{hazan2007computational} introduced a general framework based on fixed-point computation, which laid the foundation for the techniques in \cite{chen2024playing}. The idea of minimizing more general notions of regret and interpolating between them was further developed in \cite{lu2025sparsitybasedinterpolationexternalinternal}. The framework of regret circuits \cite{farina2019regretcircuitscomposabilityregret} provides a general calculus showing how local regret minimizers over simple convex sets can be composed into a regret minimizer over more complex structured sets. \cite{luo2025simultaneousswapregretminimization} study simultaneous swap regret minimization for different proper loss functions. The approach of a meta-learner that weights different regret minimizers is also common; for example, \cite{zhang2020minimizingdynamicregretadaptive} use Hedge to adaptively select the best learning rate.


%% file: 2-hidden_game.tex
\section{Preliminaries and Setting}\label{sec:setting}

In this section, we first give background on regret minimization, game theory and the oracle models that we consider in this paper, and then present the hidden game problem (\cref{prob:hidden-game}). 

\paragraph{Playing repeated games} We consider the setting in which a two-player \footnote{Most of the rest of this paper applies to multi-player games; we consider two players for simplicity of presentation.} game is played for $T$ time steps, and each player's objective is to maximize their cumulative reward. We assume that each player has action space given by $[N]$, for some (large) $N \in \mathbb{N}$; then the game is specified by matrices $A, B \in \mathbb{R}^{N \times N}$, denoting the payoffs for players 1 and 2 respectively. We let $\Delta_N = \{ x \in \mathbb{R}_{\geq 0}^N :\ x_1 + \cdots + x_N = 1 \}$ denote each player's space of mixed strategies.

\subsection{Defining the hidden game problem}\label{subsec:hidden_game}

In the \emph{hidden game problem}, player 1's payoff matrix $A \in [0,2]^{N \times N }$ has the following structure: for some constant $\rho \in (0,1)$, we can write 
$$ A = A_0 + \rho A_1 , $$
for some matrices $A_0, A_1 \in [0,1]^{N \times N}$, where $A_1$ is arbitrary and $A_0$ is given by
$$ A_0(i,j) = \mycases {1}{$i \in R$}{0}{o/w} \quad \forall i,j \in [N] $$ 
for some (unknown) subset $R \subset [N]$ with 
\(r=\left|R\right|\). In words, $A_0$ enforces that each of player 1's actions in $R$ yields better reward than all actions of player 1 outside $R$, regardless of what player \(2\) is playing. One should think of $R$ as encoding, for example, the set of all sentences that are grammatically correct. 

\paragraph{Regret minimization in repeated games.} 

Suppose that player~2 plays a sequence of mixed strategies $y_1, \ldots, y_T \in \Delta_N$ and player~1 uses an algorithm $\mathcal{A}$ to choose a sequence of strategies $x_1, \ldots, x_T \in \Delta_N$. We define the loss vector of player $1$ at time $t$ by
\[
    \ell_t := A y_t \in [0,1]^N,
\]
so that the loss incurred by player $1$ when playing $x_t$ is $\ell_t^\top x_t$. We now introduce the general notion of $\Phi$-regret, from the perspective of player~1; analogous definitions hold for player~2. 

\begin{definition}[$\Phi$-regret]
    \label{def:phi} 
    Let $\Phi$ denote a set of mappings $\phi:\Delta_N \to \Delta_N$. The $\Phi$-regret of algorithm $\mathcal{A}$ is the maximum excess loss that can be avoided by using a fixed mapping $\phi \in \Phi$ to modify each strategy $x_t$, for a worst-case choice of loss sequence $\ell_1,\ldots,\ell_T$:
\begin{align*}
    \phiregret(\mathcal{A}) 
    &=  \max_{\phi \in \Phi}  
       \sum_{t=1}^T \ell_t^\top \phi(x_t) - \sum_{t=1}^T \ell_t^\top x_t .
\end{align*}

\end{definition}

  
We consider several types of regret that are defined by various choices for the set $\Phi$.  
External regret (or ``standard regret", or just ``regret") measures the performance of the player compared to the best fixed pure strategy in hindsight. It can be realized as a special case of $\phiregret$ for the set of mappings $\Phi_{E} = \{\psi_i : i \in [N], \ \psi_i(x) = e_i \ \forall x \}$; equivalently, we have
\begin{align*}
    \text{Regret}(\mathcal{A}) 
    &= \max_{k \in [N]} \sum_{t=1}^T  \ell_t(k) - \ell_t^\top x_t .
\end{align*}

A stronger notion of regret, \emph{swap regret}, is based on the modification set $\Phi$ defined as the linear extension of all mappings from the pure action set $[N]$ to itself. Formally, for $\sigma : [N] \to [N]$, let $P_\sigma \in \{0,1\}^{N \times N}$ be the permutation matrix
\[
  (P_{\sigma})_{i,j} = 
  \begin{cases}
    1 & \text{if } i = \sigma(j), \\
    0 & \text{otherwise}.
  \end{cases}
\]
Then, for $\Phi_S := \{x \mapsto P_\sigma x : \sigma : [N] \to [N]\}$, swap regret is the case of $\phiregret$ with $\Phi = \Phi_S$: 
\begin{align*}
    \text{SwapRegret}(\mathcal{A})
    &= \max_{\phi \in \Phi_S} \sum_{t=1}^T \ell_t^\top \phi(x_t) - \sum_{t=1}^T \ell_t^\top x_t.
\end{align*}


\paragraph{Solution concepts in game theory.} 
Let $p$ be a joint distribution over the players’ actions, and let $i,j \sim p$ be the sampled actions. A coarse correlated equilibrium (CCE) satisfies, for player~I,
\[
\E_p[e_i^\top A e_j] \;\ge\; \max_{k\in [N]} \E_p[e_k^\top A e_j],
\]
and symmetrically for player~II. If all players have vanishing external regret, the empirical play converges to a CCE~\cite{foster1997calibrated,Hart00asimplecorrelated}.  

In a CCE, players cannot gain by deviating to a fixed pure strategy before seeing their recommendation, but they may still improve after conditioning on it. A correlated equilibrium (CE) rules this out: given her sampled action, no player benefits from deviating. Formally,
\[
\E_p[e_i^\top A e_j] \;\ge\; \max_{\phi \in \Phi_S} \E_p[\phi(e_i)^\top A e_j].
\]
Unlike Nash equilibrium, the distribution $p$ need not factorize across players. It is known that if all players minimize swap regret, play converges to a CE~\cite{blummansour2007fromexternaltointernal}.




\subsection{Oracle models}\label{subsec:oracles}

A \emph{pure optimization oracle} takes a history of opponent actions $j_1,\dots,j_t$ and returns
\[
\mathcal{O}^{\text{pure}}(j_1,\dots,j_t)
= \arg\max_{i \in [N]} \; e_i^\top A \sum_{s=1}^t e_{j_s}.
\]
Lower bounds~\cite{hazan2016computational} show that pure oracles alone cannot yield low-regret algorithms without runtime scaling with $N$. Following~\cite{agarwal2019learning}, we therefore consider a \emph{smooth optimization oracle}:
\[
\mathcal{O}^{\text{smooth}}(j_1,\dots,j_t)
= \arg\max_{i \in [N]} \; e_i^\top \!\Big(A \sum_{s=1}^t e_{j_s} + r \Big),
\]
where $r \in \mathbb{R}^N$ is random. We assume each oracle call takes unit time. In \cref{sec:efficient}, we discuss how Follow-the-Perturbed-Leader (FPL)~\cite{kalai2005efficient}, with a smooth oracle, achieves external regret minimization in time polynomial in $T$ and independent of $N$.

\begin{problem}[Hidden game problem]
  \label{prob:hidden-game}
Given a normal-form game $A = A_0 + \rho A_1$ (as in \cref{subsec:hidden_game}) and access to a smooth optimization oracle (as in \cref{subsec:oracles}), design an algorithm $\mathcal{A}$ that simultaneously guarantees: 
\begin{enumerate}
    \item If a hidden game exists, $\mathcal{A}$ achieves $\varepsilon T$ swap regret in at most $\poly(r,\varepsilon^{-1}, \log N)$ time steps and oracle calls.  
    \item In general, $\mathcal{A}$ achieves $\varepsilon T$ external regret in at most $\poly(\log N, \varepsilon^{-1})$ time steps and oracle calls.
\end{enumerate}
\end{problem}

\section{Swap Regret Minimization with Hidden Game}\label{sec:s-t}

As a warm-up to our full algorithm, we first assume a hidden-game exists and study the simpler task of minimizing swap regret alone. This will provide the main building blocks for Section~\ref{sec:combined}.

Formally, Algorithm~\ref{algo:st-alg} operates in the online experts setting, receiving loss vectors $\ell_1,\dots,\ell_T \in \mathbb{R}^N$. It uses a subroutine $\Aswap$ with the guarantee that for any sequence $\ell_1,\dots,\ell_{T'} \in [0,1]^{N'}$, its swap regret is bounded by $\swapregret_{N',T'}(\Aswap)$. For instance, the Blum–Mansour algorithm~\cite{blummansour2007fromexternaltointernal} satisfies $\swapregret_{N',T'}(\Aswap) = O(\sqrt{T'N' \log N'})$. 

Algorithm~\ref{algo:st-alg} maintains a growing guess $S_t \subseteq [N]$ of the hidden set $R$ and runs a swap-regret minimizer on $S_t$. We prove two key facts: (i) the support of play is always contained in $R$, and (ii) the swap regret over the entire action set $[N]$ depends only on $r = |R|$ rather than $N$. These structural properties later enable the combination of hidden-game guarantees with external regret minimization.  

The idea is simple: at each round, the algorithm augments $S_t$ by adding new actions suggested by a \emph{weighted best response}. For each $i \in S_t$, define
\[
\BR_{t_0,t}(i) = \arg\max_{j \in [N]} \; L_{t_0,t}(i)^\top e_j,
\quad\text{where}\quad
L_{t_0,t}(i) = \sum_{t'=t_0}^t x_{t'}(i)\,\ell_{t'}.
\]
Here $x_{t'}(i)$ is the probability of playing $i$ at round $t'$. In normal-form games with $\ell_t = A e_{j_t}$, the weighted best response $\BR_{t_0,t}(i)$ can be implemented using a pure optimization oracle (\cref{subsec:oracles}) by adjusting the time window and normalizing.

\begin{remark}[On the choice of $\log T$ expansions.]
We cap the number of expansion rounds at $K=\lceil \log T \rceil$, ensuring 
$|S_t|$ grows at most polynomially in $T$ (Fact~\ref{fact:support}). 
Since our swap-regret bound $O(\sqrt{T r^3 \log r})$ is sublinear whenever $r^3 \log r = o(T)$ (in particular, for $r \le T^{1/3}/(\log T)^{1/3}$), this cap does not weaken the result.
\end{remark}


\begin{algorithm}[h]
\caption{Algorithm for minimizing swap regret for the hidden game problem} \label{algo:st-alg}
\begin{algorithmic}[1]
  \STATE {\bf Input:} Algorithm $\Aswap$ with swap regret guarantee. 
  \STATE Initialize $t_0 \gets 1$ and $S_0 \gets \{ 1 \}$. ({\emph{The choice of $S_0 = \{1 \}$ is arbitrary.}})
\FOR{$t = 1, \ldots, T$}


\IF {$1<t\le \log T$}
\STATE Compute the best response to each $i \in S_{t-1}$, breaking ties arbitrarily, and add it to the support set: \[S_t \leftarrow S_{t-1} \cup \bigcup_{i \in S_{t-1}} \BR_{t_0,t-1}(i) \,.\]\label{alg-line:argmax}
\ENDIF

\IF {$S_t \neq S_{t-1} $}
\STATE Set $t_0 \gets t$. 
\STATE Restart $\Aswap$ by initializing its action set to be $S_t$. \label{alg_line:restart}
\ENDIF

\STATE Choose action $i_t$ according to algorithm $\Aswap$. \emph{(Note that we always have $i_t \in S_t$.)}
\STATE Observe $\ell_t$ and update $\Aswap$ according to the restriction of $\ell_t$ to $S_t$. 
\ENDFOR
\end{algorithmic}
\end{algorithm}

\begin{lemma}[Support containment]
  \label{lem:st-bound}
For all $t \in [T]$, the set $S_t$ in \cref{algo:st-alg} satisfies $S_t \subseteq R$. 
\end{lemma}

\begin{proof}
  Any $j$ which is added to $S_{t-1}$ at some step of the algorithm must satisfy $j \in \BR_{t_0,t-1}(i)$ for some $i \in [N]$, i.e., for some $t_0 <  t$, we have
  \begin{align}
    \sum_{t'=t_0}^{t-1} x_{t'}(i) \cdot \ell_{t'}(j) =& \max_{k \in [N]} \sum_{t' = t_0}^{t-1} x_{t'}(i) \cdot \ell_{t'}(k)\geq \sum_{t'=t_0}^{t-1} x_{t'}(i)\nonumber,
  \end{align}
  where the inequality follows by choosing any $k \in R$ which ensures $\ell_{t'}(k) \geq 1$. But if $j \not \in R$, then the left-hand side of the above display is bounded above by $\sum_{t'=t_0}^{t-1} x_{t'}(i) \cdot \rho < \sum_{t'=t_0}^{t-1} x_{t'}(i)$, which is a contradiction. Hence $j \in R$. 
\end{proof}


\begin{lemma}[Swap regret bound for \cref{algo:st-alg}] \label{lem:nfg}
The swap regret of \cref{algo:st-alg} is bounded by  $ r \cdot \swapregret_{r,T}(\Aswap) $.
\end{lemma}

\begin{proof}
  Let the time steps $t$ for which $S_t \neq S_{t-1}$ be denoted $t_{0,1}, t_{0,2}, \ldots, t_{0, K}$. (As a matter of convention, we let $t_{0,1} = 0$ and $t_{0,K+1} = T+1$.) By \cref{lem:st-bound} we have $K \leq r$. Let $J_k := [t_{0,k}, t_{0,k+1}-1] \subset [T]$. To simplify notation, we write $S_k := S_{t_{0,k}} = \cdots = S_{t_{0,k+1}-1}$. For each $k \in [K-1]$, the swap regret guarantee of $\Aswap$ gives that
  \begin{align}
\sum_{i \in S_k} \max_{i' \in S_k} \sum_{t \in J_k} x_t(i) \cdot (\ell_t(i') - \ell_t(i)) \leq \swapregret_{|J_k|,|S_k|}(\Aswap)\nonumber.
  \end{align}
  Moreover, we have, for each $k \in [K-1]$, 
  \begin{align}
\sum_{i \in S_k} \max_{i' \in [N]} \sum_{t \in J_k} x_{t}(i) \cdot (\ell_{t}(i') - \ell_{t}(i)) \leq & \sum_{i \in S_k} \max_{i' \in S_k} \sum_{t \in J_k} x_{t}(i) \cdot (\ell_{t}(i') - \ell_{t}(i))\nonumber,
  \end{align}
  since $S_t$ does not increase in size at step $t = t_{0,k+1}-1$. Combining the above displays and using that $K \leq r$, $|S_k| \leq r, |J_k| \leq T$ gives the desired result. 
\end{proof}


%% file: 3-det-alg-sim-ext-swap-new.tex
\section{Simultaneous External and Swap Regret Minimization}\label{sec:combined}

In Algorithm~\ref{algo:combined_regret}, we present a procedure that simultaneously 
minimizes external regret in general games and swap regret whenever a hidden game exists. 
The algorithm can be viewed as a modification of Algorithm~\ref{algo:st-alg} together with 
the Blum--Mansour framework~\citep{blummansour2007fromexternaltointernal} (\textsc{BM}), adapted to the 
hidden-game setting. The key idea is to maintain a current guess $S_t \subseteq [N]$ of the hidden set $R$ and to run 
$|S_t|$ copies of an external-regret algorithm $\mathcal{H}$ on this reduced space. Each copy 
outputs a distribution $q^s_t \in \Delta^{|S_t|}$, and stacking these gives a stochastic matrix 
$Q_t \in \mathbb{R}^{|S_t|\times |S_t|}$. For use in the full action space, we define its 
zero-padded extension $\widetilde Q_t \in \mathbb{R}^{N\times N}$. In parallel, the algorithm runs 
an efficient external-regret learner $\mathcal{E}$ over the full action set (instantiated as 
\textsc{FPL} with a smooth optimization oracle). At each round this produces an index 
$n_t \in [N]$, from which we form the matrix $P_t = \mathbf{1} e_{n_t}^\top \in \mathbb{R}^{N\times N}$, 
that is, all zeros except for the $n_t$-th column, which is filled with ones. A master algorithm $\calB$ then learns a convex combination 
of $P_t$ and $Q_t$, yielding a stochastic matrix $M_t$. Finally, we compute an approximate 
fixed point $x_t$ of the map $x \mapsto M_t^\top x$, which is the play of the round.


\begin{algorithm}[h!]
\caption{Combined regret minimization} \label{algo:combined_regret}
\begin{algorithmic}[1]
\STATE \textbf{Input:} algorithms $\calB,\calH$ over simplices, algorithm $\calE$ over $[N]$, constant $c$.
\STATE Initialize $\mB$ over $\Delta_2$, set $t_0 \gets 1$, $S_0 \gets \{1\}$.
\FOR{$t=1,\dots,T$}
    \IF{$1<t\le \log T$} \STATE Expand $S_t \gets S_{t-1} \cup \bigcup_{i\in S_{t-1}} \BR_{t_0,t-1}(i)$.
    \ELSE \STATE $S_t \gets S_{t-1}$.
    \ENDIF
    \IF{$S_t \neq S_{t-1}$} 
        \STATE $t_0 \gets t$; restart $\calH_1,\dots,\calH_{|S_t|}$ over $\Delta_{|S_t|}$.
    \ENDIF
    \STATE \textit{Construct $Q_t$:} each $\calH_s$ outputs $q_t^s\in\Delta_{|S_t|}$; form $Q_t^\top=\sum_s q_t^s e_s^\top$, then $\tilde Q_t$ by zero-padding.
    \STATE \textit{Construct $P_t$:} $\calE$ outputs $n_t\in[N]$; set $P_t=\mathbf 1 e_{n_t}^\top$.
    \STATE \textit{Combine:} $\beta_t \gets \mB$; set $M_t=\beta_t(1)P_t+\beta_t(2)\tilde Q_t$.
    \STATE \textit{Play:} compute $x_t$ with $\|M_t^\top x_t - x_t\|_1 \le 1/\sqrt{t}$ (See Observation \ref{lem:fixed-point}). 
    \STATE \textit{Update:} output $x_t$, observe $\ell_t$.
        \begin{itemize}
            \item Update $\calE$ with $\ell_t$.
            \item Update each $\calH_s$ with restricted loss $\ell_t^s = x_t(s)\ell_t|_{S_t}$.
            \item Update $\mB$ with $v_t(1)=x_t^\top P_t\ell_t$, $v_t(2)=x_t^\top\tilde Q_t\ell_t$.
        \end{itemize}
\ENDFOR
\end{algorithmic}
\end{algorithm}

\begin{remark}[Indexing]\label{rem:indexing}
For notational simplicity we assume that the actions in $S_t$ are indexed as $\{1,\ldots,|S_t|\}$. 
This is without loss of generality: at each update we maintain a dictionary 
$\pi_t : S_t \cup \{n_t\} \to [\,|S_t|+1\,]$ that assigns consecutive indices to the active set. 
When computing the fixed point $x_t$ we work in this reduced space; afterwards, we map back using 
$\pi_t^{-1}$. Updating $\pi_t$ takes $O(|S_t|)$ time. 
\end{remark}

\begin{fact}[Support size bound]\label{fact:support}
At most one new action is added per element of $S_{t-1}$ in each expansion step, and we perform 
at most $K=\lceil \log T\rceil$ expansions. Hence
\[
|S_t| \le \min\{r,\,2^K\} \le \min\{r,\,T\}.
\]
\end{fact}

\begin{proof}
Each expansion doubles the set size at most, and there are at most $K$ expansions. 
\lemref{lem:st-bound} implies $|S_t|\le r$, hence the bound.
\end{proof}

We now present an important observation, its proof deferred to \cref{app:analysis}:
\begin{observation}[Approximate fixed point]\label{lem:fixed-point}
An $\varepsilon$-approximate fixed point $x_t\in\Delta_t$ of 
$M_t^\top$ can be computed in $\mathrm{poly}(|S_t|,\log(1/\varepsilon))$ time, 
which guarantees $|x_t^\top M_t \ell_t - x_t^\top \ell_t|\le \varepsilon$ for all $\ell_t\in[0,1]^N$.
\end{observation}

Feasibility is immediate since both $P_t$ and $\widetilde Q_t$ are row-stochastic, so $M_t$ preserves $\Delta_t$ and admits a stationary distribution.


\subsection{Main Result}
In this section we present and prove our main result, \cref{thm:combined}. The result is stated for abstract regret minimizers $\calB,\calH,\calE$. In Section~\ref{sec:efficient} we show that these can be instantiated efficiently, 
so that the guarantees hold with per-round runtime polynomial in $T$ and independent of $N$.

\begin{theorem}\label{thm:combined}
Suppose Algorithm~\ref{algo:combined_regret} is instantiated with
\begin{itemize}
    \item a master algorithm $\calB$ over $\Delta_2$ with $O(\sqrt{T})$ external regret,
    \item base algorithms $\calH$ over $\Delta_{|S_t|}$ with $O(\sqrt{T\log |S_t|})$ external regret, and
    \item an algorithm $\calE$ over $[N]$ with $O(\sqrt{T\log N})$ external regret.
\end{itemize}
Then
\[
\mathbb{E}[\mathrm{External\mbox{-}Regret}_T(\mathcal{A})] 
= O(\sqrt{T\log N}).
\]
If there exists a hidden game of size $r$, then also
\[
\mathbb{E}[\mathrm{Swap\mbox{-}Regret}_T(\mathcal{A})] 
= O\!\left( \sqrt{T r^3 \log r}\right).
\]
\end{theorem}

We first present a technical lemma, whose proof is similar to that of \lemref{lem:nfg} and being deferred to
to \cref{app:analysis}. 

\begin{lemma}\label{lem:lts}
    If there is an hidden game, then \cref{algo:combined_regret} ran as in \cref{thm:combined} achieves
    \begin{align*}
        \E\left[\max_{\phi\in\Phi_S}\sum_{t=1}^T\ell_t^\top\phi(x_t)  - \sum_{t=1}^T \sum_{s=1}^{|S_t|}\left(\ell_t^s\right)^\top q_t^s\right]=\O\left(\sqrt{Tr^3\log r}\right)\,.
    \end{align*}
\end{lemma}
We now proceed to the proof of our main result:
\begin{proof}[of \cref{thm:combined}]
By Lemma~\ref{lem:fixed-point}, the approximate fixed point $x_t$ can be computed in 
$\mathrm{poly}(T)$ time with error $\varepsilon_t=1/\sqrt{t}$, yielding $O(\sqrt{T})$ cumulative 
error:
\begin{align*}
\E\!\left[\text{External-Regret}_T(\A)\right] 
&= \E\!\left[\max_{n\in[N]} \sum_{t=1}^T \ell_t(n) - \sum_{t=1}^T x_t^\top \ell_t \right] \\
&= \E\!\left[\max_{n\in[N]} \sum_{t=1}^T \ell_t(n) - \sum_{t=1}^T \beta_t^\top v_t \right] + O(\sqrt{T}) \,.
\end{align*}

Next, by using the regret guarantee of \textsc{Hedge} we get

\begin{align*}
\E\!\left[\text{External-Regret}_T(\A)\right] 
&\le \E\!\left[\max_{n\in[N]} \sum_{t=1}^T \ell_t(n) - \sum_{t=1}^T v_t(1)\right] + O(\sqrt{T}) 
\\
&= \E\!\left[\max_{n\in[N]} \sum_{t=1}^T \ell_t(n) - \sum_{t=1}^T x_t^\top \mathbf{1}e_{n_t}^\top \ell_t \right] + O(\sqrt{T}) \\
&= \E\!\left[\max_{n\in[N]} \sum_{t=1}^T \ell_t(n) - \sum_{t=1}^T \ell_t(n_t) \right] + O(\sqrt{T}) \,.
\end{align*}

Finally, by the external-regret bound of $\calE$, we obtain
\[
\E\!\left[\text{External-Regret}_T(\A)\right] = O(\sqrt{T\log N}) \,.
\]

\medskip

For the swap regret we argue similarly:
\begin{align*}
\E\!\left[\text{Swap-Regret}_T(\A)\right] 
&\le \E\!\left[\max_{\phi\in\Phi_S}\sum_{t=1}^T \ell_t^\top \phi(x_t) - \sum_{t=1}^T v_t(2)\right] + O(\sqrt{T}) 
\\
&
= \E\!\left[\max_{\phi\in\Phi_S}\sum_{t=1}^T \ell_t^\top \phi(x_t) - \sum_{t=1}^T x_t^\top \tilde{Q}_t \ell_t \right] + O(\sqrt{T}) \,.
\end{align*}
Using a similar argument to that in \cite{blummansour2007fromexternaltointernal} we can write
\begin{align*}
\E\!\left[\text{Swap-Regret}_T(\A)\right] 
&= \E\!\left[\max_{\phi\in\Phi_S}\sum_{t=1}^T \ell_t^\top \phi(x_t) 
       - \sum_{t=1}^T \sum_{s=1}^{|S_t|} x_t(s)\,\ell_t^\top q_t^s \right] + O(\sqrt{T}) \\
&= \E\!\left[\max_{\phi\in\Phi_S}\sum_{t=1}^T \ell_t^\top \phi(x_t) 
       - \sum_{t=1}^T \sum_{s=1}^{|S_t|} (\ell_t^s)^\top q_t^s \right] + O(\sqrt{T}) \,.
\end{align*}

By \lemref{lem:lts}, this gives the claimed swap-regret bound.

\end{proof}

\subsection{Efficient Implementation}\label{sec:efficient}

Theorem~\ref{thm:combined} assumes access to external-regret minimizers 
$\calB,\calH,\calE$ with the guarantees listed there. 
In this section we instantiate them with concrete algorithms: 
$\calB=\calH=\textsc{Hedge}$ and $\calE=\textsc{FPL}$ with smooth 
optimization oracle access, following \cite{agarwal2019learning}. This yields the regret guarantees of 
Theorem~\ref{thm:combined} with per-round runtime polynomial in $T$ and 
independent of $N$.

\begin{algorithm}[h]
\caption{Follow-the-Perturbed-Leader (FPL)}\label{alg:FPL}
\begin{algorithmic}[1]
\STATE \textbf{Input:} step size $\eta>0$, perturbation distribution $\mathcal{D}$.
\STATE Draw random vector $r\in\mathbb{R}^N$ coordinate-wise from $\mathcal{D}$.
\STATE Let $x_{1}=\arg\max_{x\in \Delta_N}\{x^\top r\}$.
\FOR{$t=1,\ldots,T$}
  \STATE Output $x_t$, receive reward $f_t(x_t)$.
  \STATE \label{line:update}
  \[
  x_{t+1} \;=\; \arg\max_{x\in \Delta_N}
  \left\{ \sum_{s=1}^{t} \nabla f_s(x_s)^\top x \;+\; r^\top x \right\}.
  \]
\ENDFOR
\end{algorithmic}
\end{algorithm}

\begin{corollary}[Kalai--Vempala~\cite{kalai2005efficient}]
\label{cor:fpl}
If $\eta=\sqrt{\frac{\log N}{T}}$ and $\mathcal{D}$ is the exponential distribution 
with density proportional to $e^{-\eta x}$, then Algorithm~\ref{alg:FPL} achieves
\[
\mathbb{E}\!\left[\mathrm{ExternalRegret}(\mathcal{A})\right]
= O\!\left(\sqrt{T\log N}\right).
\]
Moreover, only one call to the smooth optimization oracle is made per round.
\end{corollary}

\paragraph{Runtime analysis.}
The runtime contributions of the three subroutines are as follows. 
Each copy of $\calH$ operates on $\Delta_{|S_t|}$, where $|S_t|\le T^{c'}$ by 
Fact~\ref{fact:support}, or $|S_t|\le r$ under the hidden-game assumption; its per-round cost 
is therefore polynomial in $T$. The master algorithm $\calB$ operates on $\Delta_2$ and 
is constant-time. The learner $\calE$ (FPL with smooth oracle) requires a single oracle call 
per round, independent of $N$. Combining these, each round of 
Algorithm~\ref{algo:combined_regret} runs in $\mathrm{poly}(T)$ time, independent of $N$. 

\begin{corollary}
With the instantiations $\calB=\calH=\textsc{Hedge}$ and $\calE=\textsc{FPL}$ with smooth 
oracle access, Algorithm~\ref{algo:combined_regret} satisfies the guarantees of 
Theorem~\ref{thm:combined} with per-round runtime $\mathrm{poly}(T)$, independent of $N$.
\end{corollary}

%% file: 6-discussion.tex
\section{Discussion}

This work introduces the hidden game problem and provides an efficient algorithm that simultaneously minimizes external regret in general games and swap regret when hidden structure exists. Our results show that hidden structure in large games can be exploited algorithmically without sacrificing rationality, yielding convergence to correlated equilibria with runtime independent of the ambient action space. This advances the theory of learning in games with exponentially large action spaces and connects to recent motivations from AI alignment and language games.

\paragraph{Limitations.}
Our model assumes a consistency condition in which actions in the hidden set dominate those outside. This captures settings such as debate, where a subset of arguments consistently outperforms irrelevant ones, but may be restrictive in other domains. More general structures might allow conditioning on the opponent’s action or introducing random rewards outside the hidden set. In addition, our analysis focuses on adversarial swap regret minimization, which is natural for debate-style adversarial games but potentially conservative in cooperative or partially aligned settings. Most importantly, future work must move beyond theory to practice by developing scalable versions of our algorithm for large language models and evaluating them in realistic multi-agent language-game environments.

%% file: A1-deferred-proofs.tex
\section{Deferred Analysis}\label{app:analysis}

\begin{observation}[Restatement of observation \ref{lem:fixed-point}]
Let $M_t=\beta_t(1)P_t+\beta_t(2)\tilde Q_t$ be as in Algorithm~\ref{algo:combined_regret}, and 
$\Delta_t=\{x\in\Delta_N:\mathrm{supp}(x)\subseteq S_t\cup\{n_t\}\}$. 
For any $\varepsilon>0$, an $x_t\in\Delta_t$ satisfying $\|M_t^\top x_t-x_t\|_1\le\varepsilon$ 
can be computed in time $\mathrm{poly}(|S_t|,\log(1/\varepsilon))$ by solving the linear programming problem
\[
M_t^\top x = x,\quad \mathbf 1^\top x = 1,\quad x\ge 0
\]
(on the coordinates $S_t\cup\{n_t\}$) to accuracy $\varepsilon$. Moreover, for any $\ell_t\in[0,1]^N$,
\[
\bigl|x_t^\top M_t \ell_t - x_t^\top \ell_t\bigr|
\;\le\; \|M_t^\top x_t - x_t\|_1 \cdot \|\ell_t\|_\infty
\;\le\; \varepsilon.
\]
\end{observation}

\begin{proof}
Each $q_t^s\in\Delta_{|S_t|}$ is a probability vector, so $Q_t$ is row-stochastic; its zero-extension 
$\tilde Q_t$ is row-stochastic on $S_t$. Also $P_t=\mathbf 1 e_{n_t}^\top$ is row-stochastic. Hence 
$M_t$ is row-stochastic on $S_t\cup\{n_t\}$, so the map $x\mapsto M_t^\top x$ preserves $\Delta_t$ and 
admits a stationary distribution; thus the feasibility system is nonempty. Solving the linear equalities 
with the simplex constraints on $|S_t|+1$ variables to accuracy $\varepsilon$ can be done in 
$\mathrm{poly}(|S_t|,\log(1/\varepsilon))$ time using standard routines. 
The inequality follows from Hölder’s inequality and $\ell_t\in[0,1]^N$.
\end{proof}

\begin{lemma}[Lemma \ref{lem:lts} restated]
    If there is an hidden game, then \cref{algo:combined_regret} ran as in \cref{thm:combined} achieves
    \begin{align*}
        \E\left[\max_{\phi\in\Phi_S}\sum_{t=1}^T\ell_t^\top\phi(x_t)  - \sum_{t=1}^T \sum_{s=1}^{|S_t|}\left(\ell_t^s\right)^\top q_t^s\right]=\O\left(\sqrt{r^3T\log r}\right)\,.
    \end{align*}
\end{lemma}

\begin{proof}
      Let the time steps $t$ for which $S_t \neq S_{t-1}$ be denoted $t_{0,1}, t_{0,2}, \ldots, t_{0, K}$. (As a matter of convention, we let $t_{0,1} = 0$ and $t_{0,K+1} = T+1$.) By \cref{lem:st-bound} we have $K \leq r$. Let $J_k := [t_{0,k}, t_{0,k+1}-1] \subset [T]$. To simplify notation, we write $S_k := S_{t_{0,k}} = \cdots = S_{t_{0,k+1}-1}$. For each $k \in [K-1]$, consider the guarantee that we have for each algorithm \(\calH_s\), namely, for each \(s^\prime\in[|S_k|]\) we have
      \begin{align*}
          \E\left[\sum_{t\in J_k}\ell_t^s(s^\prime)-\sum_{t\in J_k}\left(\ell_t^s\right)^\top q_t^s\right]\leq\sqrt{|J_k|\log|S_k|}
      \end{align*}
      and therefore, by summing over \(s\) and taking the maximal \(s^\prime\) for each we get
        \begin{align}
\E\left[\sum_{s\in[|S_k|]} \max_{s^\prime \in [|S_k|]} \sum_{t \in J_k} x_t(s)  (\ell_t(s^\prime) - \ell_t^\top q_t^s)\right] \leq \swapregret_{|J_k|,|S_k|}(\Aswap)\nonumber.
  \end{align}
  By Remark \ref{rem:indexing}, this can be written as
  \begin{align}
\E\left[\sum_{i \in S_k} \max_{i' \in S_k} \sum_{t \in J_k} x_t(i) \cdot (\ell_t(i') - \ell_t(i))\right]\leq \swapregret_{|J_k|,|S_k|}(\Aswap)\nonumber.
  \end{align}
  Moreover, we have, for each $k \in [K-1]$, 
  \begin{align}
\E\left[\sum_{i \in S_k} \max_{i' \in [N]} \sum_{t \in J_k} x_{t}(i) \cdot (\ell_{t}(i') - \ell_{t}(i))\right] \leq & \E\left[\sum_{i \in S_k} \max_{i' \in S_k} \sum_{t \in J_k} x_{t}(i) \cdot (\ell_{t}(i') - \ell_{t}(i))\right]\nonumber,
  \end{align}
  since $S_t$ does not increase in size at step $t = t_{0,k+1}-1$. Combining the above displays and using that $K \leq r$, $|S_k| \leq r, |J_k| \leq T$ gives the desired result. 
\end{proof}

%% file: arxiv-version.bbl
\begin{thebibliography}{41}
\providecommand{\natexlab}[1]{#1}
\providecommand{\url}[1]{\texttt{#1}}
\expandafter\ifx\csname urlstyle\endcsname\relax
  \providecommand{\doi}[1]{doi: #1}\else
  \providecommand{\doi}{doi: \begingroup \urlstyle{rm}\Url}\fi

\bibitem[Agarwal et~al.(2019)Agarwal, Gonen, and Hazan]{agarwal2019learning}
Naman Agarwal, Alon Gonen, and Elad Hazan.
\newblock Learning in non-convex games with an optimization oracle.
\newblock In \emph{Conference on Learning Theory}, pages 18--29. PMLR, 2019.

\bibitem[Assos et~al.(2023)Assos, Attias, Dagan, Daskalakis, and Fishelson]{assos2023online}
Angelos Assos, Idan Attias, Yuval Dagan, Constantinos Daskalakis, and Maxwell~K Fishelson.
\newblock Online learning and solving infinite games with an erm oracle.
\newblock In \emph{The Thirty Sixth Annual Conference on Learning Theory}, pages 274--324. PMLR, 2023.

\bibitem[Aumann(1974)]{aumann1974subjectivity}
Robert~J. Aumann.
\newblock Subjectivity and correlation in randomized strategies.
\newblock \emph{Journal of Mathematical Economics}, 1\penalty0 (1):\penalty0 67--96, 1974.
\newblock \doi{10.1016/0304-4068(74)90037-8}.

\bibitem[Bighashdel et~al.(2024)Bighashdel, Wang, McAleer, Savani, and Oliehoek]{bighashdel2024policy}
Ariyan Bighashdel, Yongzhao Wang, Stephen McAleer, Rahul Savani, and Frans~A Oliehoek.
\newblock Policy space response oracles: A survey.
\newblock \emph{arXiv preprint arXiv:2403.02227}, 2024.

\bibitem[Blum and Mansour(2007)]{blummansour2007fromexternaltointernal}
Avrim Blum and Yishay Mansour.
\newblock From external to internal regret.
\newblock \emph{Journal of Machine Learning Research}, 8\penalty0 (47):\penalty0 1307--1324, 2007.
\newblock URL \url{http://jmlr.org/papers/v8/blum07a.html}.

\bibitem[Bousquet and Warmuth(2002)]{bousquet2002tracking}
Olivier Bousquet and Manfred~K Warmuth.
\newblock Tracking a small set of experts by mixing past posteriors.
\newblock \emph{Journal of Machine Learning Research}, 3\penalty0 (Nov):\penalty0 363--396, 2002.

\bibitem[Cesa-Bianchi and Lugosi(2006)]{cesa2006prediction}
Nicolo Cesa-Bianchi and G{\'a}bor Lugosi.
\newblock \emph{Prediction, learning, and games}.
\newblock Cambridge university press, 2006.

\bibitem[Chen et~al.(2009)Chen, Deng, and Teng]{settlingnash}
Xi~Chen, Xiaotie Deng, and Shang-Hua Teng.
\newblock Settling the complexity of computing two-player nash equilibria.
\newblock \emph{J. ACM}, 56\penalty0 (3), May 2009.
\newblock ISSN 0004-5411.
\newblock \doi{10.1145/1516512.1516516}.
\newblock URL \url{https://doi.org/10.1145/1516512.1516516}.

\bibitem[Chen et~al.(2024)Chen, Chen, Foster, and Hazan]{chen2024playing}
Xinyi Chen, Angelica Chen, Dean Foster, and Elad Hazan.
\newblock Playing large games with oracles and {AI} debate.
\newblock In \emph{Agentic Markets Workshop at ICML 2024}, 2024.
\newblock URL \url{https://openreview.net/forum?id=b2gfTWXG7V}.

\bibitem[Chernov and Vovk(2010)]{chernov2010predictionadviceunknownnumber}
Alexey Chernov and Vladimir Vovk.
\newblock Prediction with advice of unknown number of experts, 2010.
\newblock URL \url{https://arxiv.org/abs/1006.0475}.

\bibitem[Dagan et~al.(2023)Dagan, Daskalakis, Fishelson, and Golowich]{dagan2023external}
Yuval Dagan, Constantinos Daskalakis, Maxwell Fishelson, and Noah Golowich.
\newblock From external to swap regret 2.0: An efficient reduction and oblivious adversary for large action spaces, 2023.

\bibitem[Daskalakis et~al.(2009)Daskalakis, Goldberg, and Papadimitriou]{daskalakis2009complexity}
Constantinos Daskalakis, Paul~W Goldberg, and Christos~H Papadimitriou.
\newblock The complexity of computing a nash equilibrium.
\newblock \emph{Communications of the ACM}, 52\penalty0 (2):\penalty0 89--97, 2009.

\bibitem[Farina et~al.(2019)Farina, Kroer, and Sandholm]{farina2019regretcircuitscomposabilityregret}
Gabriele Farina, Christian Kroer, and Tuomas Sandholm.
\newblock Regret circuits: Composability of regret minimizers, 2019.
\newblock URL \url{https://arxiv.org/abs/1811.02540}.

\bibitem[Foster and Vohra(1997)]{foster1997calibrated}
Dean~P Foster and Rakesh~V Vohra.
\newblock Calibrated learning and correlated equilibrium.
\newblock \emph{Games and Economic Behavior}, 21\penalty0 (1-2):\penalty0 40, 1997.

\bibitem[Freund and Schapire(1999)]{freund1999adaptive}
Yoav Freund and Robert~E. Schapire.
\newblock Adaptive game playing using multiplicative weights.
\newblock \emph{Games and Economic Behavior}, 29\penalty0 (1–2):\penalty0 79--103, 1999.
\newblock ISSN 0899-8256.
\newblock \doi{10.1006/game.1999.0738}.
\newblock URL \url{https://www.sciencedirect.com/science/article/pii/S0899825699907388}.

\bibitem[Gemp et~al.(2022)Gemp, Anthony, Bachrach, Bhoopchand, Bullard, Connor, Dasagi, De~Vylder, Duenez-Guzman, Elie, et~al.]{gemp2022developing}
Ian Gemp, Thomas Anthony, Yoram Bachrach, Avishkar Bhoopchand, Kalesha Bullard, Jerome Connor, Vibhavari Dasagi, Bart De~Vylder, Edgar~A Duenez-Guzman, Romuald Elie, et~al.
\newblock Developing, evaluating and scaling learning agents in multi-agent environments.
\newblock \emph{AI Communications}, 35\penalty0 (4):\penalty0 271--284, 2022.

\bibitem[Globus-Harris et~al.(2023)Globus-Harris, Harrison, Kearns, Roth, and Sorrell]{globus-harris2023multicalibration}
Ira Globus-Harris, Declan Harrison, Michael Kearns, Aaron Roth, and Jessica Sorrell.
\newblock Multicalibration as boosting for regression.
\newblock In \emph{Proceedings of the 40th International Conference on Machine Learning (ICML)}, Honolulu, Hawaii, USA, 2023. JMLR.org.
\newblock \doi{10.48550/arXiv.2301.13767}.
\newblock URL \url{https://arxiv.org/abs/2301.13767}.

\bibitem[Gofer et~al.(2013)Gofer, Cesa-Bianchi, Gentile, and Mansour]{pmlr-v30-Gofer13}
Eyal Gofer, Nicolò Cesa-Bianchi, Claudio Gentile, and Yishay Mansour.
\newblock Regret minimization for branching experts.
\newblock In Shai Shalev-Shwartz and Ingo Steinwart, editors, \emph{Proceedings of the 26th Annual Conference on Learning Theory}, volume~30 of \emph{Proceedings of Machine Learning Research}, pages 618--638, Princeton, NJ, USA, 12--14 Jun 2013. PMLR.
\newblock URL \url{https://proceedings.mlr.press/v30/Gofer13.html}.

\bibitem[Hannan(1957)]{hannan1957approximation}
James Hannan.
\newblock Approximation to bayes risk in repeated play.
\newblock In M.~Dresher, A.~W. Tucker, and P.~Wolfe, editors, \emph{Contributions to the Theory of Games}, volume~3, pages 97--139. 1957.

\bibitem[Hart and Mas-Colell(2000)]{Hart00asimplecorrelated}
Sergiu Hart and Andreu Mas-Colell.
\newblock A simple adaptive procedure leading to correlated equilibrium.
\newblock \emph{Econometrica}, 68\penalty0 (5):\penalty0 1127--1150, 2000.
\newblock \doi{https://doi.org/10.1111/1468-0262.00153}.
\newblock URL \url{https://onlinelibrary.wiley.com/doi/abs/10.1111/1468-0262.00153}.

\bibitem[Hazan(2016)]{hazan2016introduction}
Elad Hazan.
\newblock Introduction to online convex optimization.
\newblock \emph{Foundations and Trends{\textregistered} in Optimization}, 2\penalty0 (3-4):\penalty0 157--325, 2016.

\bibitem[Hazan and Kale(2007)]{hazan2007computational}
Elad Hazan and Satyen Kale.
\newblock Computational equivalence of fixed points and no regret algorithms, and convergence to equilibria.
\newblock \emph{Advances in Neural Information Processing Systems}, 20, 2007.

\bibitem[Hazan and Koren(2016)]{hazan2016computational}
Elad Hazan and Tomer Koren.
\newblock The computational power of optimization in online learning.
\newblock In \emph{Proceedings of the forty-eighth annual ACM symposium on Theory of Computing}, pages 128--141, 2016.

\bibitem[Hazan et~al.(2016)Hazan, Koren, Livni, and Mansour]{hazan2016onlinelearninglowrank}
Elad Hazan, Tomer Koren, Roi Livni, and Yishay Mansour.
\newblock Online learning with low rank experts, 2016.
\newblock URL \url{https://arxiv.org/abs/1603.06352}.

\bibitem[Irving et~al.(2018)Irving, Christiano, and Amodei]{irving2018debate}
Geoffrey Irving, Paul Christiano, and Dario Amodei.
\newblock Ai safety via debate.
\newblock \emph{arXiv preprint arXiv:1805.00899}, 2018.

\bibitem[Kalai and Vempala(2005)]{kalai2005efficient}
Adam Kalai and Santosh Vempala.
\newblock Efficient algorithms for online decision problems.
\newblock \emph{Journal of Computer and System Sciences}, 71\penalty0 (3):\penalty0 291--307, 2005.

\bibitem[Kleinberg et~al.(2023)Kleinberg, Paes~Leme, Schneider, and Teng]{kleinberg2023ucalibration}
Bobby Kleinberg, Renato Paes~Leme, Jon Schneider, and Yifeng Teng.
\newblock U-calibration: Forecasting for an unknown agent.
\newblock In Gergely Neu and Lorenzo Rosasco, editors, \emph{Proceedings of the Thirty Sixth Conference on Learning Theory}, volume 195 of \emph{Proceedings of Machine Learning Research}, pages 5143--5145. PMLR, December 2023.
\newblock \doi{10.48550/arXiv.2307.00168}.
\newblock URL \url{https://arxiv.org/abs/2307.00168}.

\bibitem[Lu et~al.(2025)Lu, Sun, and Zhang]{lu2025sparsitybasedinterpolationexternalinternal}
Zhou Lu, Y.~Jennifer Sun, and Zhiyu Zhang.
\newblock Sparsity-based interpolation of external, internal and swap regret, 2025.
\newblock URL \url{https://arxiv.org/abs/2502.04543}.

\bibitem[Luo et~al.(2025)Luo, Senapati, and Sharan]{luo2025simultaneousswapregretminimization}
Haipeng Luo, Spandan Senapati, and Vatsal Sharan.
\newblock Simultaneous swap regret minimization via kl-calibration, 2025.
\newblock URL \url{https://arxiv.org/abs/2502.16387}.

\bibitem[Mansour et~al.(2022)Mansour, Mohri, Schneider, and Sivan]{mansour2022strategizing}
Yishay Mansour, Mehryar Mohri, Jon Schneider, and Balasubramanian Sivan.
\newblock Strategizing against learners in bayesian games.
\newblock In Po-Ling Loh and Maxim Raginsky, editors, \emph{Proceedings of the Thirty Fifth Conference on Learning Theory}, volume 178 of \emph{Proceedings of Machine Learning Research}, pages 5221--5252. PMLR, February 2022.
\newblock \doi{10.48550/arXiv.2205.08562}.
\newblock URL \url{https://arxiv.org/abs/2205.08562}.

\bibitem[McAleer et~al.(2021)McAleer, Lanier, Wang, Baldi, and Fox]{mcaleer2021xdo}
Stephen McAleer, John~B Lanier, Kevin~A Wang, Pierre Baldi, and Roy Fox.
\newblock Xdo: A double oracle algorithm for extensive-form games.
\newblock \emph{Advances in Neural Information Processing Systems}, 34:\penalty0 23128--23139, 2021.

\bibitem[McMahan et~al.(2003)McMahan, Gordon, and Blum]{mcmahan2003planning}
H~Brendan McMahan, Geoffrey~J Gordon, and Avrim Blum.
\newblock Planning in the presence of cost functions controlled by an adversary.
\newblock In \emph{Proceedings of the 20th International Conference on Machine Learning (ICML-03)}, pages 536--543, 2003.

\bibitem[Mourtada and Maillard(2017)]{mourtada2017efficienttrackinggrowingnumber}
Jaouad Mourtada and Odalric-Ambrym Maillard.
\newblock Efficient tracking of a growing number of experts, 2017.
\newblock URL \url{https://arxiv.org/abs/1708.09811}.

\bibitem[Muller et~al.(2021)Muller, Rowland, Elie, Piliouras, Perolat, Lauriere, Marinier, Pietquin, and Tuyls]{muller2021learning}
Paul Muller, Mark Rowland, Romuald Elie, Georgios Piliouras, Julien Perolat, Mathieu Lauriere, Raphael Marinier, Olivier Pietquin, and Karl Tuyls.
\newblock Learning equilibria in mean-field games: Introducing mean-field psro.
\newblock \emph{arXiv preprint arXiv:2111.08350}, 2021.

\bibitem[Nash(1951)]{nashnoncooperativegames}
John Nash.
\newblock Non-cooperative games.
\newblock \emph{Annals of Mathematics}, 54\penalty0 (2):\penalty0 286--295, 1951.
\newblock ISSN 0003486X, 19398980.
\newblock URL \url{http://www.jstor.org/stable/1969529}.

\bibitem[Nisan et~al.(2007)Nisan, Roughgarden, Tardos, and Vazirani]{NisaRougTardVazi07}
Noam Nisan, Tim Roughgarden, \'Eva Tardos, and Vijay~V. Vazirani.
\newblock \emph{Algorithmic Game Theory}.
\newblock Cambridge University Press, New York, NY, USA, 2007.

\bibitem[Papadimitriou(1994)]{10.1016/S0022-0000(05)80063-7}
Christos~H. Papadimitriou.
\newblock On the complexity of the parity argument and other inefficient proofs of existence.
\newblock \emph{J. Comput. Syst. Sci.}, 48\penalty0 (3):\penalty0 498–532, June 1994.
\newblock ISSN 0022-0000.
\newblock \doi{10.1016/S0022-0000(05)80063-7}.
\newblock URL \url{https://doi.org/10.1016/S0022-0000(05)80063-7}.

\bibitem[Peng and Rubinstein(2023)]{peng2023fast}
Binghui Peng and Aviad Rubinstein.
\newblock Fast swap regret minimization and applications to approximate correlated equilibria.
\newblock \emph{arXiv preprint arXiv:2310.19647}, 2023.

\bibitem[Ramalingam et~al.(2025)Ramalingam, Kiyani, and Roth]{ramalingamrelationship}
Ramya Ramalingam, Shayan Kiyani, and Aaron Roth.
\newblock The relationship between no-regret learning and online conformal prediction.
\newblock In \emph{Forty-second International Conference on Machine Learning}, 2025.

\bibitem[Stoltz and Lugosi(2005)]{stoltz2005internal}
Gilles Stoltz and Gábor Lugosi.
\newblock Internal regret in on-line portfolio selection.
\newblock \emph{Machine Learning}, 59\penalty0 (1-2):\penalty0 125--159, 2005.
\newblock \doi{10.1007/s10994-005-0465-4}.

\bibitem[Zhang et~al.(2020)Zhang, Lu, and Yang]{zhang2020minimizingdynamicregretadaptive}
Lijun Zhang, Shiyin Lu, and Tianbao Yang.
\newblock Minimizing dynamic regret and adaptive regret simultaneously, 2020.
\newblock URL \url{https://arxiv.org/abs/2002.02085}.

\end{thebibliography}
